\documentclass[runningheads]{llncs}
\usepackage[T1]{fontenc}
\usepackage{graphicx}

\usepackage[draft]{packages/duerrmacros}
\usepackage[font=small,labelfont=bf]{caption}

\usepackage{color}

\begin{document}
\title{Discovering Locally Maximal Bipartite Subgraphs}
\author{Dominik Dürrschnabel\inst{1,2}\orcidID{0000-0002-0855-4185} \and
  Tom Hanika\inst{1,2}\orcidID{0000-0002-4918-6374} \and
  Gerd Stumme\inst{1,2}\orcidID{0000-0002-0570-7908}}
\authorrunning{D. Dürrschnabel et al.}

\institute{Knowledge \& Data Engineering Group,
  University of Kassel, Kassel,  Germany\\
  \and
  Interdisciplinary Research Center for Information System Design, Kassel, Germany\\
  \email{\{duerrschnabel, hanika, stumme\}@cs.uni-kassel.de} }

\maketitle              %
\begin{abstract}

  Induced bipartite subgraphs of maximal vertex cardinality are an essential concept for the analysis of graphs.
  Yet, discovering them in large graphs is known to be computationally hard.
  Therefore, we consider in this work a weaker notion of this problem, where we discard the maximality constraint in favor of inclusion maximality.
  Thus, we aim to discover locally maximal bipartite subgraphs.
  For this, we present three heuristic approaches to extract such subgraphs and compare their results to the solutions of the global problem.
  For the latter, we employ the algorithmic strength of fast SAT-solvers.
  Our three proposed heuristics are based on a greedy strategy, a simulated annealing approach, and a genetic algorithm, respectively.
  We evaluate all four algorithms with respect to their time requirement and the vertex cardinality of the discovered bipartite subgraphs on several benchmark datasets.
  
\end{abstract}

\section{Introduction}

Graphs are used in a variety of sciences to model and to analyze complex relationships.
In this context, the search for interesting and relevant substructures is a standard procedure.
So-called bipartite graphs represent a particularly interesting substructure, since they allow a multitude of further mathematically intrinsic as well as scientifically extrinsic interpretations, e.g., communities in social networks or codewords in coding theory.
Yet, discovering the largest bipartite subgraph of a graph is a problem that is known to be of high computational complexity.
Formally, the problem that relates to our work can be stated as follows:
\emph{For a graph $G=(V,E)$, what is a subset of vertices $X\subset V$ such that the graph $G(X,E\cap (X\times X))$ is bipartite and the size of $X$ is maximal?}
From a result by Lewis and Yannakakis~\cite{Lewis.1980} follows that deciding for  $G$ whether such a set of cardinality $k$ exists is an $\NP$-complete task.
While the computational complexity of this problem is well-investigated, it is not yet explored how to efficiently discover large, yet not maximal, induced bipartite subgraphs of a given graph.%

The present work investigates a related problem where the global maximality condition of the computed subgraph is replaced with inclusion-maximality of the vertex set.
In this setting, we optimize on the order of the computed subgraph, \ie, we maximize the number of its vertices.
In order to make our algorithms suitable for real-world data, we limit ourselves to designs that are executable in polynomial time in the size of the data.

As a benchmark, we recollect an exact solution of the bipartite subgraph problem which we proposed in Dürrschnabel et al.~\cite{dimdraw}.
This solution shows a principal way how a corresponding SAT problem can be formulated and subsequently solved using a high-performance SAT-solver.
We then present three polynomial-time heuristics to compute \emph{locally maximal solutions}.
All three have different runtime-performance trade-offs which we point out.
In detail, we evaluate all four algorithms against each other on common benchmark graph datasets and compare their runtime and the qualities of the results, i.e., the cardinality of the computed induced bipartite subgraphs.
Our results show that the greedy heuristic has the best runtime, while the genetic algorithm performs best with respect to the order of the computed bipartite subgraph.
Finally, the employed simulated annealing algorithm balances runtime and performance.

We investigate this task because of our previous work~\cite{dimdraw} in the realm of order diagram drawing.
In this, the proposed solution requires the computation of a large inclusion-maximal bipartite subgraph in the so-called transitive incompatibility graph of the ordered set.
Other possible fields are datasets that are by their nature bipartite such as author-to-publication networks but are extracted from non-bipartite relations and thus have noise that makes them non-bipartite.
Also in the realm of system biology and medicine, bipartite graphs are of interest~\cite{Pavlopoulos.2018} and these areas could thus profit from our work.
Furthermore, the approach could be used to discover hidden two-mode networks in graph data.

\section{Globally Maximal Bipartite Subgraphs}
In this work, we refer to \emph{graphs} of \emph{order} $n$ as tuples $G=(V,E)$ with $|V| = n$ and $E \subset (V \times V)$ with \emph{vertex set}  $V$ and \emph{edge set} $E$.
For a subset of the vertices $X \subset V$ the graph $(X, E\cap (X \times X))$ is called an \emph{induced subgraph} of $G$ and denoted by $G[X]$. In this work we always refer to induced subgraphs when we say subgraph.
A graph is called \emph{bipartite}, if its vertices can be partitioned into two sets $A$ and $B$ with $V= A \cup B$ such that there is no edge $\{u,v\}$ with $u\in A \wedge v \in A$ or $u\in B \wedge v \in B$.
A well-known property of graphs is that they are bipartite if and only if they contain no cycles of odd length.
For a natural number $n$ we denote the set $\{i \in \N\mid 1 \leq i \leq n\}$ by $[n]$.
We only consider finite graphs, \ie, graphs of order $n < \infty$.

Even though we are interested in bipartite subgraphs of maximal order, we use the equivalent dual formulation from here on.
For a graph $G=(V,E)$, we want to find a minimal set of vertices $D$ such that the graph $G[V \setminus D]$ is bipartite.
This is, as $D$ can be interpreted as erroneous data that makes the graph non-bipartite. %

In this section we repeat and expand on an approach that we already briefly suggested in \cite{dimdraw}, to approaches the global problem formulation, \ie, the problem:

\begin{problem}
\label{prob1}
For a graph $G=(V,E)$, what is a subset of vertices $D \subset V$ such that the graph $G[V\setminus D]$ is bipartite and the size of $D$ is minimal.
\end{problem}

Because of the focus of this work, we are able to properly evaluate this approach against the later proposed heuristics.
Checking for a graph if it is bipartite can be done in polynomial time by doing a breath-first search and coloring the vertices in alternating colors conditional on their distance to the starting vertex.
Either the graph is bipartite, and the two color classes will result in a valid bipartition or the algorithm will try to assign some vertex to both color classes.
A naive approach could thus check for all subset of vertices, whether their deletion makes the graph bipartite.
However, even for small examples, \ie, graphs with few vertices, this is infeasible as checking all subset of cardinality $k$ of a graph of order $n$ will result in $\binom{n}{k}$ such tests.
For a more sophisticated approach to the problem, we reduce it to an instance of the Boolean satisfiability problem which we can then solve with a SAT-Solver, in our case MiniSat~\cite{Een.2003} in version 2.2.
To be exact, we want to know for a graph $G=(V,E)$ on $n$ vertices and $m$ edges whether by deleting $k$ vertices we can make the graph bipartite.
Solving is done by finding a partition of $V$ into the three sets $A,B,D$, such that $A$ and $B$ are independent sets and $|D|\leq k$.

\begin{definition}
  \label{CNF}
  Let $G=(V,E)$ be a graph on $n$ vertices. For each $v_i\in V$ define the three variables $V_{i,1},V_{i,2},V_{i,3}$ with the clauses $(i) V_{i,1} \vee V_{i,2} \vee V_{i,3}$ for all $v_i \in V$
  and $(ii)$ $\neg V_{i,1} \vee \neg V_{j,1}$ and $\neg V_{i,2} \vee \neg V_{j,2}$ for all $\{v_i,v_j\} \in E$.
  Furthermore, add variables and clauses such the set $\{V_{i,3} \mid i \in [n]\}$ satisfies an at-most-$k$ condition.
\end{definition}

For the \emph{at-most-k} condition, we make use of the variables and clauses introduced by Sinz~\cite{Sinz.2005}.
Altogether, our SAT instance has $(n-1)(k+3)+3$ variables and $2m+2nk+2n-3k-1$ clauses. For this CNF the following holds. %

\begin{proposition}
  The CNF from \cref{CNF} of a graph $G=(V,E)$ has a valid assignment, iff the graph can be made bipartite by deleting $k$ vertices.
  Then, for $D=\{v_i \mid V_{i,3}=T\}$, $A=\{v_i \mid V_{i,1}=T\}$ and $B=\{v_i \mid V_{i,2}=T \vee V_{i,1}=F\}$ the graph $G[V\setminus D]$ is bipartite with bipartition  classes $A$ and $B$.
\end{proposition}

\begin{proof}
  Let the conjugative normal form have a valid assignment.
  Then the set $D$ has cardinality at most $k$ by construction.
  Note, that every vertex that is not in $D$ has to be in exactly one of $A$ or $B$ because of their definition and condition $(i)$ of the conjugative normal form.
  Assume $A$ and $B$ are not bipartition classes of $G[V \setminus D]$.
  Then there have to be two vertices in either $A$ or $B$ that are connected by an edge, without loss of generality $A$.
  But this is a contradiction to the definition of $A$ and condition $(ii)$ of the conjugative normal form.
  Assume now on the other hand, that there is a set $\tilde{D}$ of cardinality at most $k$, such that $G[V \setminus D]$ is bipartite.
  Call the bipartition sets $\tilde{A}$ and $\tilde{B}$.
  But then the variable assignment $V_{i,1}=T \iff v_i \in \tilde{A}$ and $V_{i,2}=T \iff v_i \in \tilde{B}$ and $V_{i,3}=T \iff v_i \in \tilde{D}$ is a valid assignment of the CNF.
\end{proof}

Now we can build this SAT instance for each $k$ increasing from 1 until we achieve satisfiable instance.
Then, the set $\{v_i \mid V_{i,3}=T \}$ is exactly the subset of vertices we have to remove to make the graph bipartite.
As a speed-up technique we apply a binary search where at the beginning $k$ is doubled in each step until an initial solution is found.
Then, in each step the mean value of the known upper and lower bounds is checked until the exact solution is found.

\section{Heuristics for the Local Problem}

In this section we now deal with the local version of the problem. Formally:

\begin{problem}
For a graph $G=(V,E)$, what is a set of vertices $D \subset V$ such that the graph $G[V\setminus D]$ is bipartite and there is no set $\tilde{D}\subsetneq D$ with the same property.
\end{problem}

Compared to Problem~\ref{prob1}, this problem does not search for a set $D$ of global minimality and restricts it to only inclusion-minimality.

\subsection{Greedy}
\label{sec:greedy}

Our greedy algorithm is formalized in \cref{alg:greedy}, which reflects the main routine \texttt{greedy} and the subroutine \texttt{greedy-fill}.
\begin{algorithm}[t]
  \caption{Greedy Algorithm}
  \label{alg:greedy}
  \textbf{Input:} Graph $G=(V,E)$\\
  \textbf{Output:} Inclusion-minimal set $D \subset V$, such that $G[V \setminus D]$ is bipartite\\
  \phantom{\textbf{Output:}} Bipartition classes $A$ and $B$ of $G[V \setminus D]$
  \hrule
  \begin{lstlisting}
def greedy$(V,E)$:
    return greedy-fill$(V,E,\{\},\{\},V)$

def greedy-fill$(V,E,A,B,D)$:
    for $u$ in random-order$(D)$:
        $\mathcal{N} = \{N \in \{A,B\}\mid \nexists v \in N, \{u,v\} \in E\}$
        if $\mathcal{N} \neq \emptyset$:
            $D.$remove$(u)$
            random-element$(\mathcal{N}).$add($u$)
    return $A,B,D$
  \end{lstlisting}
\end{algorithm}
The subroutine has to be called with the vertex set $V$ partitioned into $A$, $B$, and $D$.
It is required that neither $A$ nor $B$ contain two vertices connected by an edge.
The algorithm checks for all vertices $u$ in $D$ in a random order, whether they already have a neighbor in one of the bipartition classes $A$ or $B$.
Then, $\mathcal{N}$ contains the classes to which $u$ can be added based on the previous assignments.
If $\mathcal{N}$ is empty, it can't be added to any bipartition class and is therefore stays in $D$, otherwise one of the classes in $\mathcal{N}$ is selected at random.
To avoid ambiguity, if $\mathcal{N}$ only contains the empty set we add $u$ to $A$.
Thus, the routine moves elements from $D$ to $A$ and $B$, until all elements in $D$ are connected with an edge to some element in $A$ and in $B$.

To compute an inclusion-minimal set $D$, such that the induced graph on $V\setminus D$ is bipartite, we can now call this subroutine with $V$ as $D$ and empty sets for $A$ and $B$.
The algorithm then computes $D$ and the bipartition classes $A$ and $B$ of the resulting graph.
By this design, we can rely on this subroutine for the heuristics later proposed in this work.

\subsection{Simulated Annealing}

Simulated annealing algorithms are motivated by the physical process of annealing metal which involves controlled cooling in order to achieve better physical properties.
Classical hill-climbing algorithms generate a starting solution which is then improved iteratively by choosing neighboring solutions of better quality.
By design, these algorithms often terminate in a minimum which is only local.
The simulated annealing approach tries to overcome this issue by accepting to worsen the current solution using a probability function which thus allows the algorithm to leave local minima.
At the beginning, the algorithm accepts worse solutions with a high probability and then decreases this probability iteratively towards zero based on a cooling function.
Thus, in the last iterations, the algorithm behaves similar to a hill-climbing version.

\begin{algorithm}[t]
  \caption{Simulated Annealing}
  \label{alg:simulatedannealing}
  \begin{tabular}{lll}
    \textbf{Input:} & Graph $G=(V,E)$                             & Maximal number of iterations $i_{max}$                            \\
                    & Starting temperature $t_{max}$ \hspace{1em} & Cooling function $cooling(i_{max}, t_{max},i)\mapsto \R_{\geq 0}$ \\
  \end{tabular}\\
  \textbf{Output:} Inclusion-minimal set $D \subset V$, such that $G[D]$ is bipartite\\
  \phantom{\textbf{Output:}} Bipartition classes $A$ and $B$ of $G[D]$
  \hrule
  \begin{lstlisting}
def simulated-annealing$(V,E,i_{max})$:
    $A,B,D\coloneq$greedy$(V,E)$
    for $i$ in $[i_{max}]$:
        $A_{c},B_{c},D_{c} \coloneq $compute-neighbor$(V,E,A,B,D)$
        $c\coloneq|D|-|D_{c}|$
        if $c>0$:
            $A,B,D \coloneq A_c,B_c,D_c$
        else:
            $t \coloneq$cooling$(i_{max}, t_{max},i)$
            if $e^{c/t}>$random$(0,1)$:
                $A,B,D \coloneq A_c,B_c,D_c$
    return $A,B,D$

def compute-neighbor$(V,E,A,B,D)$:
    $u \coloneq$random-element$(D)$
    for $v$ in neighbors$(u)$:
        $D.$add$(v)$
    $D.$remove$(u)$
    random-element$(\{A,B\}).$add$(u)$
    return greedy-fill$(V,E,(A\setminus D),(B\setminus D),D)$
  \end{lstlisting}
\end{algorithm}

Our version of the simulated approach in \cref{alg:simulatedannealing} is initialized with a maximal number of iterations $i_{max}$, a starting temperature $t_{max}$ and a cooling function which maps to a range from $0$ to $t_{max}$.
The initial solution is generated using the greedy algorithm.
In each iteration, a neighboring solution is chosen by selecting a random vertex $u$ from $D$ and removing all vertices that are connected by an edge to $u$ from the bipartition classes $A$ and $B$.
Then, $u$ can be added to $A$ or $B$, from which one is chosen at random.
Finally, the \texttt{greedy-fill} routine is used to ensure that the set $D$ is minimal.
Note, that using a cooling function that maps every value to zero results in a hill-climbing algorithm.

\subsection{Genetic Algorithm}

Finally, we propose an adaptation of a genetic algorithm which is motivated by the evolutionary process of reproduction.
For this a set of starting individuals is chosen.
The ones, with the highest fitness, with respect to the optimization task, are allowed to reproduce to generate a new generation of individuals.
Furthermore, mutations can be introduced to avoid local minima.
This process is repeated, until a good solution is found.

\begin{algorithm}[t]
  \caption{Genetic Algorithm}
  \label{alg:genetic}
  \begin{tabular}{lll}
    \textbf{Input:} & Graph $G=(V,E)$                                   & Maximal number $i_{max}$ of individuals \\
                    & Probability $p_{mut}$ for mutations \hspace*{1em} & Maximal number $g_{max}$ of generation  \\
  \end{tabular}\\
  \textbf{Output:} Inclusion-minimal set $D \subset V$, such that $G[D]$ is bipartite\\
  \phantom{\textbf{Output:}} Bipartition classes $A$ and $B$ of $G[D]$
  \hrule
  \begin{lstlisting}
def genetic$(V,E,g_{max}, i_{max})$:
    $cur\_gen\coloneq [\,]$
    for $i$ in $[i_{max}]$:
        $cur\_gen.$add$($greedy-fill$(V,E))$
    for $i$ in $[g_{max}]$:
        $next\_gen\coloneq [\,]$
        $P\coloneq$generate-probability-distribution$(cur\_gen)$
        for $i$ in $[i_{max}]$:
            $(A_1,B_1,D_1),(A_2,B_2,D_2) \coloneq$choose-with-prob$(cur\_gen,P, 2)$
            $next\_gen.$add$($breed$(V,E,A_1,B_1,D_1,D_2,p_{mut}))$
        $cur\_gen\coloneq next\_gen$
    return (A,B,D) from $cur\_gen$ with $|D|$ minimal

def breed$(V,E,A,B,D_1,D_2,p_{mut})$:
    $D_n \coloneq D_1 \cup D_2$
    if random$[0,1) < p_{mut}$:
        return compute-neighbor$(V,E,(A \setminus D_n),(B \setminus D_n),D_n)$
    else:
        return greedy-fill$(V,E,(A \setminus D_n), (B \setminus D_n), D_n)$
\end{lstlisting}
\end{algorithm}

In our version of this approach, which is formalized in \cref{alg:genetic}, we provide a maximal number of generations $g_{max}$, a number of individuals $i_{max}$ that exist in each generation and a mutation probability $p_{mut}$.
At the beginning we initialize $i_{max}$ starting individuals using the greedy algorithm from \cref{sec:greedy}.
In each iteration of the loop, a new generation $next\_gen$ of individuals is generated which replaces the current set of individuals $cur\_gen$ with a new one.
For this, a probability distribution based on the fitness of the current generation is computed.
We choose the probability distribution such that the probability is linear with respect to the cardinality of $D$ and that the element with the smallest cardinality $D$ has ten times the probability of the one with the largest cardinality $D$.
The elements of $next\_gen$ are generated by choosing two distinct elements for each based on the probability distribution and breeding them to a new individual.
This is done by using the \texttt{breed} function which first computes the set $D_n$ as the union of both $D$-sets from the chosen individuals.
Because bipartiteness is a hereditary property and $D_n$ is a superset of the $D$-sets, the graph $G[V\setminus D_n]$ is also bipartite.
Finally, the new individual is generated using the routine \texttt{greedy-fill} to ensure that the resulting $D$ is inclusion-minimal.
Alternatively, with probability $p_{mut}$, a mutation is introduced using the neighborhood choosing from the simulated annealing algorithm.

\section{Evaluation and Discussion}

In this section we discuss and compare the different approaches proposed in this work with respect to their runtime and the quality of their results.

\subsection{Datasets and Implementation}
\label{sec:datasets}

We test our algorithm on four different datasets.
The DAGmar\footnote{\url{https://www.infosun.fim.uni-passau.de/~chris/down/MIP-1202.pdf}} generator can produce random leveled graphs for given vertex numbers and densities.
We use the graphs bundled with the DAGmar generator as our first testing dataset.
Those have between 20 and 400 vertices and an edge-vertex ratio between 1.6 and 10.6.
We refer to this dataset as \emph{DAGmar}.
The \emph{Rome} and \emph{North} (AT\&T) graphs\footnote{\url{http://www.graphdrawing.org/data.html}} are two datasets that are well-known benchmark graphs amd are commonly applied by the graph drawing community.
The \emph{Random} class contains randomly generated graphs on vertices between 10 and 500 vertices. The edge-vertex ratio was chosen, such that it is between 1 and 10.

All of our experiments are implemented in Python 3, for the SAT-solver we used the binaries provided on the MiniSat website\footnote{\url{http://minisat.se/MiniSat.html}}.
The experiments are conducted on an Intel Xeon Gold 5122 CPU equipped with 300 GB RAM.
For reproducibility  and further research, our source code is public.\footnote{\url{https://github.com/domduerr/bipartite}}

\subsection{Parameter Tuning}
\label{sec:parameter-tuning}

Our simulated annealing approach and our genetic algorithm require parameters which we have to chosen carefully.
We provide recommendations based on extensive optimizations.
Due to space constraints, we are only able to describe those partially in this section.

\paragraph{Simulated Annealing}
\label{sec:simulated-annealing}

\begin{figure}[t]
  \begin{minipage}{0.6\linewidth}
    \captionof{table}{Different cooling functions for the simulated annealing algorithm.}
    \label{tab:siman_cooling}
    \begin{tabular}[t]{llr}
      \toprule
      Name           & $cooling(i_{m}, t_{m}, t)$                                                          & Avg. result \\
      \midrule
      Hill Climbing: & $0$                                                                                 & 20.50       \\
      Linear:        & $t_m\left(\frac{i_m-t}{i_m}\right)$                                                 & 19.87       \\
      Quadratic:     & $t_m\left(\frac{i_m-t}{i_m}\right)^2$                                               & 19.39       \\
      Exponential:   & $t_m\left(\frac{1}{1+e^{\frac{2\ln(t_m)}{i_m}\left(t-\frac{i_m}{2}\right)}}\right)$ & 19.91       \\
      \bottomrule
    \end{tabular}
  \end{minipage}%
  \begin{minipage}{0.35\linewidth}
    
    \includegraphics[width=.9\linewidth]{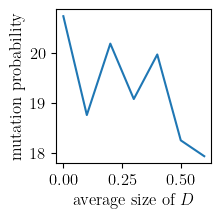}
    
    \caption{Average results for different mutation probabilities.}
    \label{fig:mut}
    
  \end{minipage}
\end{figure}

The simulated annealing algorithm has three parameters that to optimize.
For the starting temperature we observe that the average cardinality of $D$ for all graphs in the test sets decreases until a temperature of around 45.
Thus, we recommend using a starting temperature of 50.
To decide, on the cooling function we provide the average result of the tests on four different cooling functions in \cref{tab:siman_cooling}.
The results suggest, that the quadratic cooling function should be plugged.
Finally, the number of iterations that algorithm runs for can be optimized.
Our observations suggest that the improvements after more than 10000 iterations are marginal, and we thus recommend this parameter.

\paragraph{Genetic}
The genetic algorithm has three parameters.
As a general rule, we observe that increasing the number of individuals or the number of generations improves the result.
We are not able to observe a maximum number of individuals or generations, beyond which the cardinality of $D$ does not become smaller, we thus cap our algorithm at 20 individuals and 1000 generations.
For the mutation probability, \cref{fig:mut} suggests that a mutation rate of 1 provides the best results.
We also experimented with survivors between generations, however, we observed that those increases the size of $D$.

\subsection{Evaluation}
\label{sec:evaluation}

We first consider the theoretical runtime of the different algorithms.
Even though the transformation to the SAT instance is polynomial, the SAT-solver has by its nature an exponential runtime.
The greedy algorithm has to process every edge once and is thus linear in the number of edges of the graph.
In the worst case, the simulated annealing algorithm does a full procession of the greedy algorithm in each repetition and has thus a runtime of $O(m\cdot i_{max})$.
Finally, the runtime of the genetic algorithm is in $O(m\cdot g_{max}\cdot i_{max})$, as in every repetition of the loop the greedy algorithm is called for each individual.

\begin{table}[h]
  \centering
  \caption{Average result and standard deviation of the four algorithms on all graph classes. *The algorithm did not finish in under one hour for all graphs.}
  \label{tab:res}
  \begin{tabular}{lrrrr}
    \toprule
    Avg.~Result & Dagmar         & North       & Random         & Rome        \\
    \midrule
    SAT         & 18.12 (10.07)* & 2.02 (2.89) & 14.65 (10.39)* & 3.93 (2.36) \\
    Greedy      & 114.37 (72.34) & 2.05 (2.94) & 127.44 (91.11) & 5.84 (3.81) \\
    Sim.~Ann.   & 97.81 (64.43)  & 2.02 (2.89) & 108.75 (80.13) & 4.16 (2.52) \\
    Genetic     & 93.32 (61.91)  & 1.83 (7.46) & 103.44 (76.81) & 3.95 (2.38) \\
    \bottomrule
  \end{tabular}

  \caption{Average time consumption and standard deviation of the four algorithms on all graph classes.}
  \label{tab:time}
  \begin{tabular}{lrrrr}
    \toprule
    Avg.~Time & Dagmar             & North           & Random              & Rome            \\
    \midrule
    SAT       & 273.93s (858.47s)* & 0.23s (2.49s)   & 420.51s (1055.62s)* & 0.46s (2.41s)   \\
    Greedy    & 4.51ms (3.67ms)    & 0.26ms (0.17ms) & 4.77ms (3.72ms)     & 0.25ms (0.12ms) \\
    Sim.~Ann. & 9.39s (6.16s)      & 0.66s (0.68s)   & 9.54s (6.75s)       & 0.98s (0.3s)    \\
    Genetic   & 324.43s (217.8s)   & 35.88s (15.41s) & 306.97s (210.01s)   & 43.51s (7.16s)  \\
    \bottomrule
  \end{tabular}
\end{table}

For the experiment-based evaluation, we choose our parameters as described in the last section.
We fist consider the experimental runtime.
Each graph is processed in our test-datasets with each algorithm once.
If the SAT-algorithm ran longer than one hour, we canceled its computation.
For the DAGmar graph class, only the computation of 304 from 1960 graphs finished in this time, for the random class it was 89 of 403.
In \cref{tab:time,tab:res} we compare the average runtime and size $D$ of the different algorithms split by graph class.
If we compare the heuristics, it is as expected by the theoretical runtime investigation:
The greedy algorithm is the fastest one and the genetic algorithm is the slowest with the simulated annealing algorithm being in the middle.
When we consider the size of $D$, it is the other way round: The best results are achieved by the genetic algorithm and the worst ones by the greedy algorithm on all test datasets.
The SAT algorithm is, by its nature of being an exact algorithm, always the one with the smallest set $D$.
\begin{figure}[b]
  \centering
  \includegraphics[width=\textwidth]{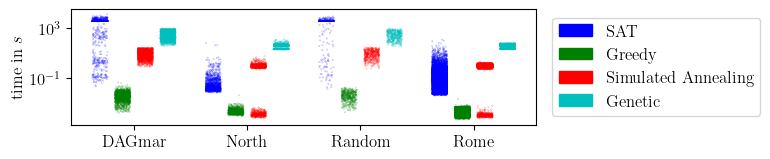}
  \caption{Time required by the different algorithms. Each dot represents a graph in the given class.}
  \label{fig:time}
\end{figure}
A more visual representation of the time requirement is provided in the plot of \cref{fig:quality}.
Here each graph in each graph class is represented by a dot which is placed on a logarithmic scale in the $y$-axis, depending on the time its computation needs for each algorithm.
In \cref{fig:quality} we plot for all graph classes and algorithms, how many graphs had a computed set $D$ that is larger than $k$.
To provide comparability, the graphs that did not finish in the SAT-experiment are excluded from this plot.
Once again, we can observe that the greedy algorithm performs worst, while the simulated annealing and genetic algorithm come very close to the SAT-algorithm which, by its nature of being an exact algorithm, is always the lowest curve in all four plots.

\begin{figure}[t]
  \centering
  \null\hfill\includegraphics[width=0.48\linewidth]{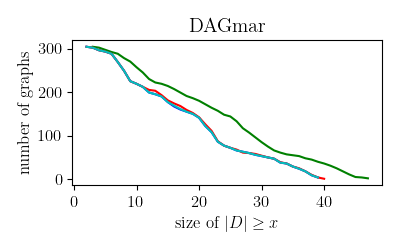}\hfill%
  \includegraphics[width=0.48\linewidth]{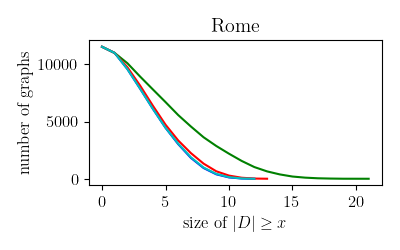}\hfill\null\\
  \null\hfill\includegraphics[width=0.48\linewidth]{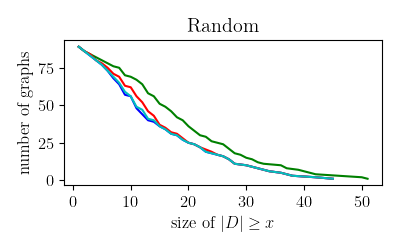}\hfill%
  \includegraphics[width=0.48\linewidth]{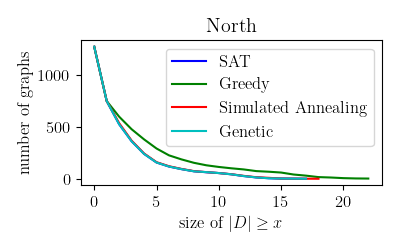}\hfill\null
  \caption{Number of graphs that where the size of the computed $D$ is at least $x$ for the different algorithms.}
  \label{fig:quality}
\end{figure}

\subsection{Discussion}
\label{sec:discussion}

We propose four different algorithms in this work.
Each of these algorithms has its merits, as they balance runtime against cardinality of $D$ in different ways.
An exact result is usually preferred, however we observe that the exact SAT-algorithm does not finish in under an hour for medium-sized and larger graphs.
Because of the exponential nature of SAT-solving, one should thus not expect to always find a solution in a reasonable time.
On the other hand, the greedy algorithm has a runtime which is linear in the number of edges of the graph and can thus usually be computed even on large graphs.
The order of the computed bipartite subgraphs are not too unreasonable when compared to the exact solution of the SAT algorithm.
Finally, the simulated annealing and greedy algorithm are in the middle of the other two approaches in both regards, runtime and quality.
The simulated annealing approach is a bit faster but delivers worse results.
However, it should be noted that both of them are very close to the exact solutions.
For a general recommendation, one should usually make use of the algorithm with the highest runtime one can afford.

\section{Related Work}

Bipartiteness of graphs is a so-called hereditary property, \ie, every subgraph of a bipartite graph is once again bipartite.
Lewis and Yannakakis~\cite{Lewis.1980} show in their work from 1980 that for any hereditary property $P$, it is $\NP$-complete to decide, whether a graph $G$ can be made to satisfy $P$ by deleting $k$ vertices.
Even approximations are known to be in the same complexity class~\cite{Lund.1993}.

Cohen et al.~\cite{Cohen.2008} provide an algorithm to compute all maximal induced subgraphs with a desired hereditary property and Trukhanov et al.~\cite{Trukhanov.2013} give a framework on how to design an algorithm to find maximum vertex subset~\cite{Trukhanov.2013} with given hereditary property.
Because of the complexity of the underlying decision problem, both approaches result in exponential algorithms.

A closely related problem is the edge-deletion problem.
Here, instead of vertices, a set of edges is requested that makes the graph bipartite if it is deleted.
Similarly to the vertex-deletion problem, deciding on the number of edges that have to be deleted is $\NP$-complete~\cite{Yannakakis.1981}.

Another well-investigated problem about subgraphs with a hereditary property is the clique-problem.
For a given graph, the largest fully-connected subgraph is sought-after.
Its $\NP$-completeness was first proven by Karp in 1972~\cite{Karp.1972}.
To tackle this problem, there are exact algorithms~\cite{Bron.1973} as well as heuristics \cite{Wu.2015}.

Deciding if an instance of the Boolean satisfiability problem is solvable is another $\NP$-complete~\cite{Cook.1971} task.
So-called SAT-solvers such as \texttt{MiniSat}~\cite{Een.2003} try to investigate instances of this problem and, if the instance is solvable, try to provide valid variable assignments in reasonable time.
In this work, we recapped our previous approach~\cite{dimdraw} which makes use of this by deducing the largest bipartite induced subgraph of a graph from the assignment computed by a SAT-solver for an instance of the problem.

Furthermore, we proposed three heuristics that compute locally maximal subgraphs: a greedy algorithm, a simulated annealing algorithm and a genetic algorithm.
Simulated annealing and genetic algorithms are well-established tools in algorithm engineering.
We refer the reader to surveys~\cite{Suman.2006,Katoch.2021} to get an overview on the variations of both approaches.

\section{Conclusion}

In this work we proposed three algorithms to compute locally maximal induced bipartite subgraphs of large order. To this end, we employed an exact solution using a SAT-solver and three different heuristic approaches with a greedy strategy, a simulated annealing approach and a genetic algorithm.
Furthermore, we compared the results on four benchmark datasets and demonstrated in an experimental evaluation that all three heuristics have a reason to exist by balancing the time-consumption and the order of the computed subgraph to different degrees.

A notable observation is, that the simulated annealing approach and the genetic heuristics could, by their design, work on any hereditary graph property.
For this, only the greedy heuristic and the neighbor-choosing have to be adjusted.
This raises the question, whether the approach generalizes to other properties of graphs that are hereditary.
To confirm that these approaches work on other classes, further studies are necessary and are a natural future extension to this work.
As we published our code, such experiments can easily be conducted.

\bibliographystyle{template/splncs04}
\bibliography{paper}

\end{document}